\documentclass[runningheads]{llncs}

\usepackage{graphicx}
\usepackage{comment}
\usepackage{amsmath,amssymb} % define this before the line numbering.
\usepackage{color}
\usepackage{cite}
\usepackage{multirow}
\usepackage{hyperref}
\usepackage[table,xcdraw]{xcolor}
\usepackage{subcaption}
\usepackage[misc,geometry]{ifsym}

\graphicspath{{.}{images/}}

 \usepackage[mathscr]{eucal}

\usepackage{tabularx}
\usepackage{multirow}
\usepackage{enumerate}

\usepackage{colortbl}
\usepackage{dcolumn}
\newcolumntype{.}{D{.}{.}{1.3}}

\usepackage{txfonts}
\usepackage[subnum]{cases}

\newcommand\dashrule{\leavevmode\xleaders\hbox{-}\hfill\kern0pt}

\def\diag{\operatorname{diag}}

\newcommand{\ba}{{\bvec{a}}}
\newcommand{\bb}{\bvec{b}}
\newcommand{\bc}{\bvec{c}}

\newcommand{\bw}{\bvec{w}}

\newcommand{\bA}{{\bf A}}
\newcommand{\bB}{{\bf B}}
\newcommand{\bC}{{\bf C}}

\newcommand{\bI}{{\bf I}}

\newcommand{\bK}{{\bf K}}

\newcommand{\bQ}{{\bf Q}}

\newcommand{\bU}{{\bf U}}
\newcommand{\bV}{{\bf V}}
\newcommand{\bW}{{\bf W}}

\newcommand{\bZ}{{\bf Z}}

\newcommand{\be}{\begin{eqnarray}}
\newcommand{\ee}{\end{eqnarray}}

\newcommand{\matrixb}{\left[ \begin{array}}
\newcommand{\matrixe}{\end{array} \right]}

\newcommand{\tr}{\mathop{\rm tr}\nolimits}

\def\*{\circledast}

\newcommand{\bvec}[1]{\boldsymbol{#1}}

\def\vectorize{\operatorname{vec}}
\newcommand{\vtr}[1]{\vectorize\hspace{-.3ex}\left(#1\right)}

\newcommand{\tensor}[1]{\boldsymbol{\mathscr{\MakeUppercase{#1}}}} %tensor

\newcommand{\tE}{\tensor{E}}

\newcommand{\tG}{\tensor{G}}

\newcommand{\tK}{\tensor{K}}

\newcommand{\tT}{\tensor{T}}

\newcommand{\tX}{\tensor{X}}
\newcommand{\tY}{\tensor{Y}}

\newcommand{\minitab}[2][l]{\begin{tabular}{@{}#1}#2\end{tabular}}
\usepackage[vlined,ruled,commentsnumbered,algo2e]{algorithm2e}

\usepackage{aurical}
\usepackage{arydshln}

\usepackage{algorithm}
\usepackage{algorithmic}

\graphicspath{{.}{images/}}

\begin{document}
% \linenumbers
\pagestyle{headings}
\mainmatter
\def\ECCVSubNumber{6795}  % Insert your submission number here

\title{Stable Low-rank Tensor Decomposition for Compression of Convolutional Neural Network} \footnotetext{Accepted to ECCV2020}

\begin{comment}
\titlerunning{ECCV-20 submission ID \ECCVSubNumber} 
\authorrunning{ECCV-20 submission ID \ECCVSubNumber} 
\author{Anonymous ECCV submission}
\institute{Paper ID \ECCVSubNumber}
\end{comment}
%******************

% \begin{comment}
\titlerunning{Stable Tensor Decomposition for Compression of CNN}
\author{Anh-Huy Phan$^{\textrm{\Letter}}$\inst{1} \and
Konstantin Sobolev\inst{1} \and
Konstantin Sozykin\inst{1} \and
Dmitry Ermilov\inst{1} \and
Julia Gusak\inst{1} \and
Petr Tichavsk{\'y}\inst{2} \and
Valeriy Glukhov\inst{3} \and 
Ivan Oseledets\inst{1} \and
Andrzej Cichocki\inst{1}}
\authorrunning{A.-H. Phan et al.}
\institute{Skolkovo Institute of Science and Technology (Skoltech), Moscow, Russia \\
\email{\{a.phan, konstantin.sobolev, konstantin.sozykin, dmitrii.ermilov, y.gusak, i.oseledets, a.cichocki\}@skoltech.ru}\\
\and
The Czech Academy of Sciences, Institute of Information Theory and Automation, Prague\\
\email{tichavsk@utia.cas.cz}
\and
Noah's Ark Lab, Huawei Technologies\\
\email{glukhov.valery@huawei.com}
}

% \end{comment}
%******************
\maketitle
% Anh-Huy Phan
\begin{abstract}

Most state-of-the-art deep neural networks are overparameterized and exhibit a high computational cost. A straightforward approach to this problem is to replace convolutional kernels with its low-rank tensor approximations, whereas the Canonical Polyadic tensor Decomposition is one of the most suited models. However, fitting the convolutional tensors by numerical optimization algorithms often encounters diverging components, i.e., extremely large rank-one tensors but canceling each other. Such degeneracy often causes the non-interpretable result and numerical instability for the neural network ﬁne-tuning. This paper is the first study on degeneracy in the tensor decomposition of convolutional kernels. We present a novel method, which can stabilize the low-rank approximation of convolutional kernels and ensure efficient compression while preserving the high-quality performance of the neural networks. We evaluate our approach on popular CNN architectures for image classification
and show that our method results in much lower accuracy degradation and provides consistent performance.

\end{abstract}

\section{Introduction}
%\cite{LeCunl989}
Convolutional neural networks (CNNs) and their recent extensions have significantly increased their ability to solve complex computer vision tasks, such as image classification, object detection, instance segmentation, image generation, etc. Together with big data and fast development of the internet of things, CNNs bring new tools for solving computer science problems, which are intractable using classical approaches. 

Despite the great successes and rapid development of CNNs, most modern neural network architectures contain a huge number of parameters in the convolutional and fully connected layers, therefore, demand extremely high computational costs \cite{rigamonti2013}, which makes them difficult to deploy on devices with limited computing resources, like PC or mobile devices. Common approaches to reduce redundancy of the neural network parameters are: structural pruning \cite{He_2018_ECCV, zhuang2018discrimination, gao2018dynamic, he2018soft}, sparsification \cite{HanPrunning15, figurnov2016perforatedcnns, molchanov2017variational}, quantization \cite{rastegari2016xnor,bulat2019matrix}
%, bulat2019xnor
and low-rank approximation \cite{Denton2014,lebedev2014speeding, Kim2016,chen2018adaptive,gusak2019automated,kossaifi2019efficient}.

The weights of convolutional and fully connected layers are usually overparameterized and known to lie on a low-rank subspace \cite{Denil2013}. Hence, it is possible to represent them in low-rank tensor/tensor network formats using e.g., Canonical Polyadic decomposition (CPD) \cite{Denton2014,lebedev2014speeding,AstridL17}, Tucker decomposition \cite{Kim2016,gusak2019automated}, Tensor Train decomposition \cite{Novikov2015,Wang20203dtt}. The decomposed layers are represented by a sequence of new layers with much smaller kernel sizes, therefore, reducing the number of parameters and computational cost in the original model.

Various low-rank tensor/matrix decompositions can be straightforwardly applied to compress the kernels. This article intends to promote the simplest tensor decomposition model, the Canonical Polyadic decomposition (CPD).

\subsection{Why CPD}\label{subs::whycpd}

In neural network models working with images, the convolutional kernels are usually tensors of order 4 with severely unbalanced dimensions, e.g., $D \times D \times S \times T$, where 
$D\times D$ represents the filter sizes, $S$ and $T$ denote the number of input and output channels, respectively. The typical convolutional filters are often of relatively small sizes, e.g., $3\times 3$, $7 \times 7$, compared to the input 
% ($I$) and output ($O$) 
($S$) and output ($T$) 
dimensions, which in total may have hundred of thousands of filters.  This leads to excessive redundancy among the kernel filters,  which are particularly suited for tensor decomposition methods. 
Among low-rank tensor decomposition and tensor networks, the Canonical Polyadic tensor decomposition \cite{harshman1970foundations, hillar2013most} is the simplest and elegant model, which represents a tensor by sum of rank-1 tensors\footnote{Rank-1 tensor of size $n_1\times n_2\times\dots \times n_{d}$ is an outer product of $d$ vectors with dimensions~$n_1, n_1,\dots, n_d$.} or equivalently by factor matrices interconnected through a diagonal tensor (Fig.~\ref{fig:tucker2cp}). 
The number of parameters for a CP model of rank-$R$ is $R(2D+S+T)$ or $R(D^2+S+T)$ when we consider  kernels as order-4 tensors or their reshaped order-3 versions, respectively.
Usually, CPD gains a relatively high compression ratio since the decomposition rank is not very large \cite{lebedev2014speeding, gusak2019automated}. 

Representation of the high order convolutional kernels in the form of the CP model is equivalent to the use of separable convolutions. 
In \cite{kossaifi2019efficient}, the authors modeled the high order kernels in the generalized multiway convolution by the CP model. 

The Tucker tensor decomposition (TKD) \cite{tucker1963implications} is an alternative tensor decomposition method for convolutional kernel compression \cite{Kim2016}. The TKD provides more flexible interaction between the factor matrices through a core tensor, which is often dense in practice (Fig.~\ref{fig:tucker2}). 
Kim et al. \cite{Kim2016} investigated low-rank models at the most suited noise level for different unfoldings\footnote{The mode-$j$ unfolding of an order-$d$ tensor of size $n_1\times n_2 \times \dots \times n_d$ reorders the elements of the tensor into a matrix with $n_j$ rows and $n_1\dots n_{j - 1}n_{j + 1}\dots n_d$ columns.} of the kernel tensor. This heuristic method does not consider a common noise level for multi modes and is not optimal to attain the approximation error bound.

Block tensor decomposition \cite{Lath-BCM12} is an extension of the TKD, which models data as the sum of several Tucker or Kruskal terms, i.e., a TKD with block-diagonal core tensor. For the same multilinear rank as in TKD, BTD exhibits a smaller number of parameters; however, there are no available proper criteria for block size selection (rank of BTD). 

In addition, the other tensor networks, e.g., Tensor Train\cite{OseledetsTT09} or Tensor Chain (TC) \cite{Espig_2011,Khoromskij-SC}, are not applicable unless the kernel filters are tensorized to higher orders.
Besides, the Tensor Chain contains a loop, is not closed and leads to severe numerical instability to ﬁnd the best approximation, see Theorem 14.1.2.2\cite{Landsberg} and \cite{Handschuhth}.

We later show that CPD can achieve much better performance with an even higher compression ratio by further compression the Tucker core tensors by solving a suitably formulated optimization problem.

\begin{figure}[t]
% \vspace{-2em} 
\begin{subfigure}[t]{.28\linewidth}
\includegraphics[width=\linewidth]{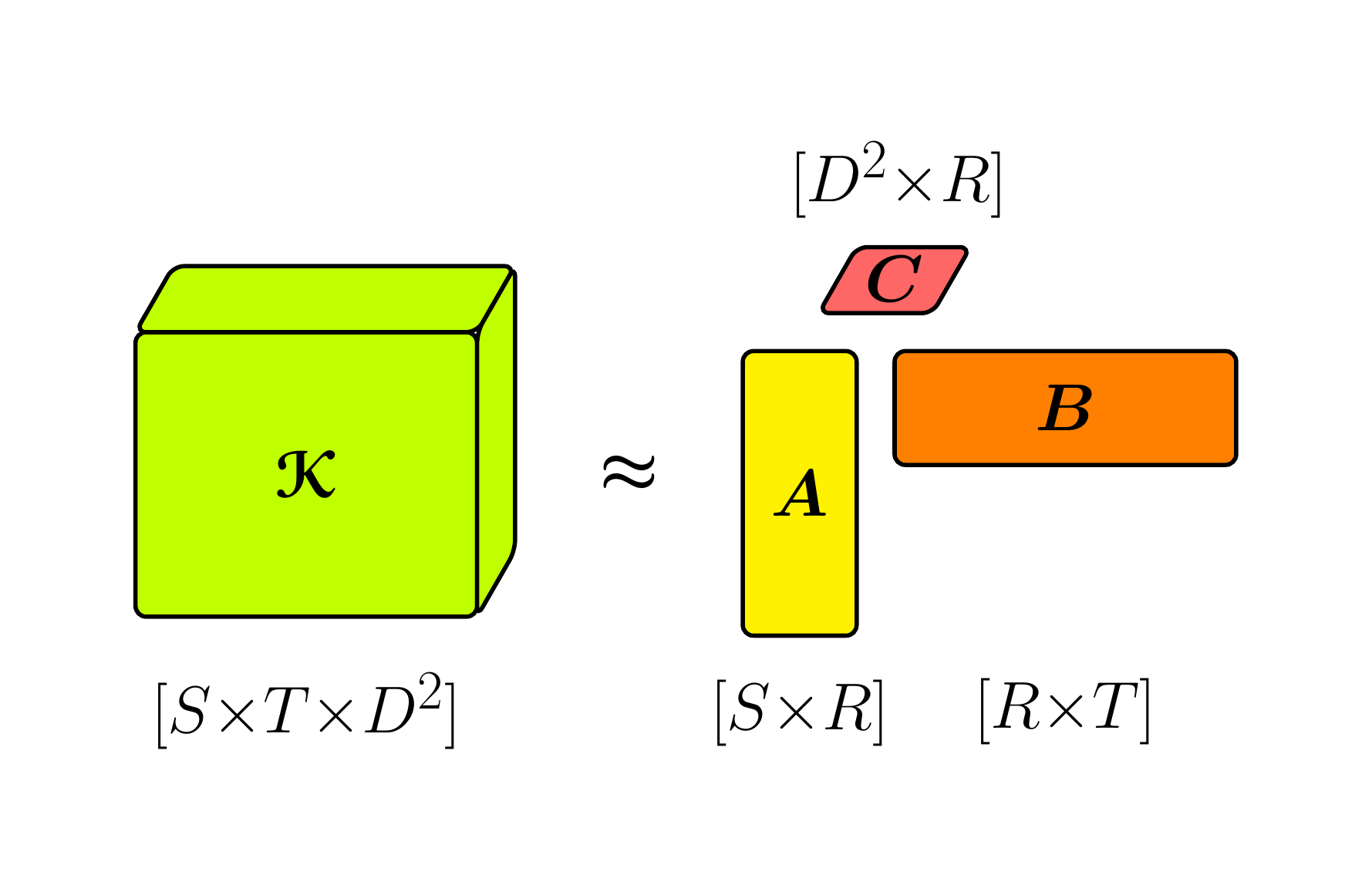}
\caption{CPD}\label{fig:tucker2cp}
\end{subfigure}
\begin{subfigure}[t]{.33\linewidth}
\includegraphics[width=\linewidth]{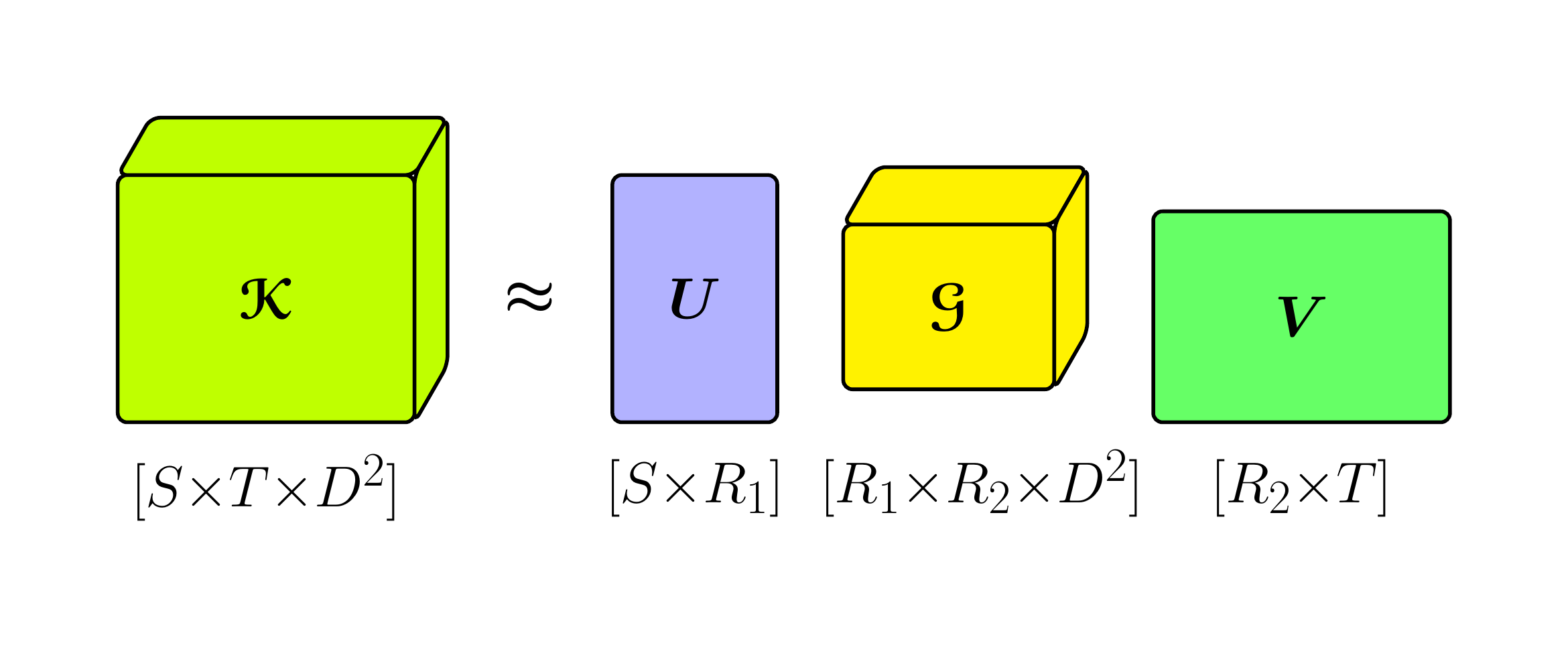}
\caption{TKD}\label{fig:tucker2}
\end{subfigure}
\begin{subfigure}[t]{.35\linewidth}
\includegraphics[width=\linewidth]{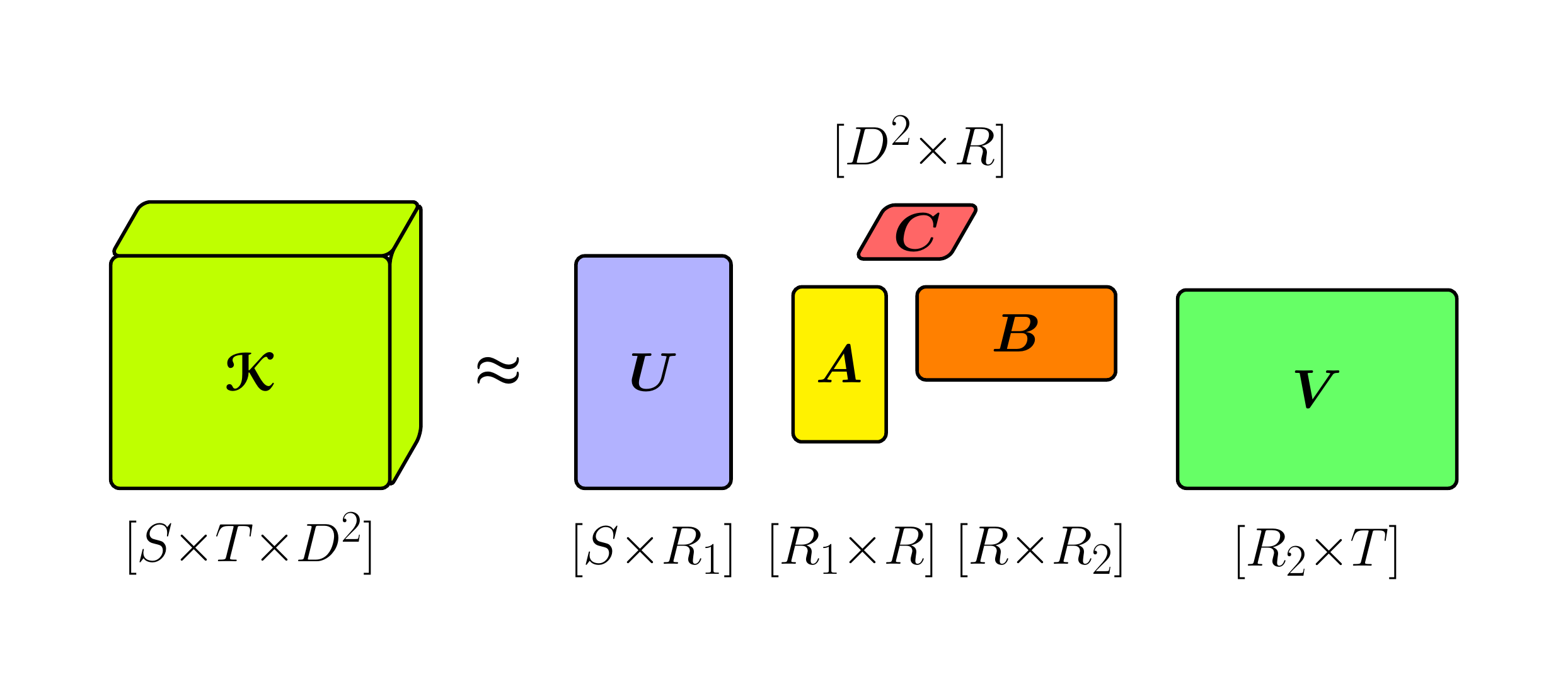}
\caption{TKD-CPD}\label{fig:tucker2epc}
\end{subfigure}
\caption{Approximation of a third-order tensor using Canonical Polyadic tensor decomposition (CPD), Tucker-2 tensor decomposition (TKD), and their combination (TKD-CPD). CPD and TKD are common methods applied for CNN compression.}\label{fig_models}
\end{figure}

\subsection{Why not Standard CPD}\label{subs::whynotcpd}

In one of the first works applying CPD to convolutional kernels, Denton et al. \cite{Denton2014} computed the CPD by sequentially extracting the best rank-1 approximation in a greedy way. 
This type of deflation procedure is not a proper way to compute CPD unless decomposition of orthogonally decomposable tensors \cite{doi:10.1137/S0895479899352045}
or with a strong assumption, e.g., at least two factor matrices are linearly independent, and the tensor rank must not exceed any dimension of the tensor \cite{Phan_tensordeflation_alg}. The reason is that subtracting the best rank-1 tensor does not guarantee to decrease the rank of the tensor \cite{Stegeman:2010:SBR}.

In \cite{lebedev2014speeding}, the authors approximated the convolution kernel tensors using the Nonlinear Least Squares (NLS) algorithm \cite{Vervliet2016tensorlab}, one of the best existing algorithms for CPD. 
%An important observation reported in the Ph.D. thesis \cite{lebedev2019phd} is that they “failed to find a good SGD learning rate” for fine-tuning, and suggesting that it is nontrivial to optimize the factorization for even a  single layer in the ImageNet models. 
However, as mentioned in the Ph.D. thesis~\cite{lebedev2019phd},  it is not trivial to optimize a neural network even when weights from a single layer are factorized, and the authors “failed. " to find a good SGD learning rate” with fine-tuning a  classification model on the ILSVRC-12 dataset.
%A similar factorization method for compression was applied in \cite{AstridL17}. 

{\bf Diverging Component - Degeneracy.} 
%A common phenomena when attempting to approximate a tensor of relatively high rank by a low-rank model or a tensor which has nonunique CPD using numerical optimization algorithms is that there should exist at least two rank-one tensors such that their (Frobenius) norms or intensities are relatively high but cancel each other
Common phenomena when using numerical optimization algorithms  to  approximate a tensor of relatively high rank by a low-rank model or a tensor, which has nonunique CPD, is that there should exist at least two rank-one tensors such that their Frobenius norms or intensities are relatively high but cancel each other
\cite{deSilva-Lim08}, $\|\mathbf {a}_{r}^{(1)}\circ \mathbf {a} _{r}^{(2)}\circ \cdots \circ \mathbf {a} _{r}^{(d)}\|_{F}\to \infty\,.$

%In convolutional kernels, To achieve low approximation error, tensor of convolutional kernel should be approximated with a rank that is relatively high comparing to the smallest dimension's size. 

The degeneracy of CPD is reported in the literature, e.g., in\cite{MITCHELL94,Paattero00,Harshman04,krijnen,cichocki2016tensor,Rayens}.
Some efforts which impose additional constraints on the factor matrices can improve stability and accelerate convergence, such as, column-wise orthogonality \cite{Rayens,krijnen}, positivity or nonnegativity \cite{limcomNN}. 
%However, the constraints may not be applicable in some data and cause the estimator not to get lower approximation error, a kind of trade-off between estimation stability and good approximation error.
However, the constraints are not always applicable in some data, and thus prevent the estimator from getting lower approximation error, yielding to the trade-off between estimation stability and good approximation error.\footnotetext{As shown in  \cite{DBLP:conf/cvpr/VasilescuT03}, RMS error is not the only one minimization criterion for a particular computer vision task.}

We have applied CPD approximations for various CNNs and confirm that the diverging component occurs for most cases when we used either Alternating Least Squares (ALS) or NLS \cite{Vervliet2016tensorlab} algorithm. %Many rank-1 tensor components have very high Frobenius norms, while some other ones are with low intensity. 
%As an example, in Fig.~\ref{fig:norm_of_lambdas}(left) we show the norm of each rank-1 tensor component in CPD with rank-500 of a convolutional layer $4^{th}$ block in Resnet-18.
As an example, we approximated one of the last convolutional layers from ResNet-18 with rank-500 CPD and plotted in Fig.~\ref{fig:norm_of_lambdas}(left) intensities of CPD components, i.e., Frobenius norm of rank-1 tensors.
The ratio between the largest and smallest intensities of rank-1 tensors was greater than 30. Fig.~\ref{fig:norm_of_lambdas}(right) shows that the sum of squares of intensities for CPD components is (exponentially) higher when the decomposition is with a higher number of components. Another criterion, sensitivity (Definition~\ref{def:sensitivity}), shows that the standard CPD algorithms are not robust to small perturbations of factor matrices, and sensitivity increases with higher CP rank.
%, while for CPD-EPC this ratio is about 1.6

Such degeneracy causes the instability issue when training a CNN with decomposed layers in the CP (or Kruskal) format.
%when training the CNN with Kruskal- structured (or CPD) convolutional layers.
More specifically, it causes difficulty for a neural network to perform fine-tuning, selecting a good set of parameters, and maintaining stability in the entire network.
This problem has not been investigated thoroughly. To the best of our knowledge, there is no method for handling this problem.

\begin{figure}[t]
\includegraphics[width=0.48\linewidth]{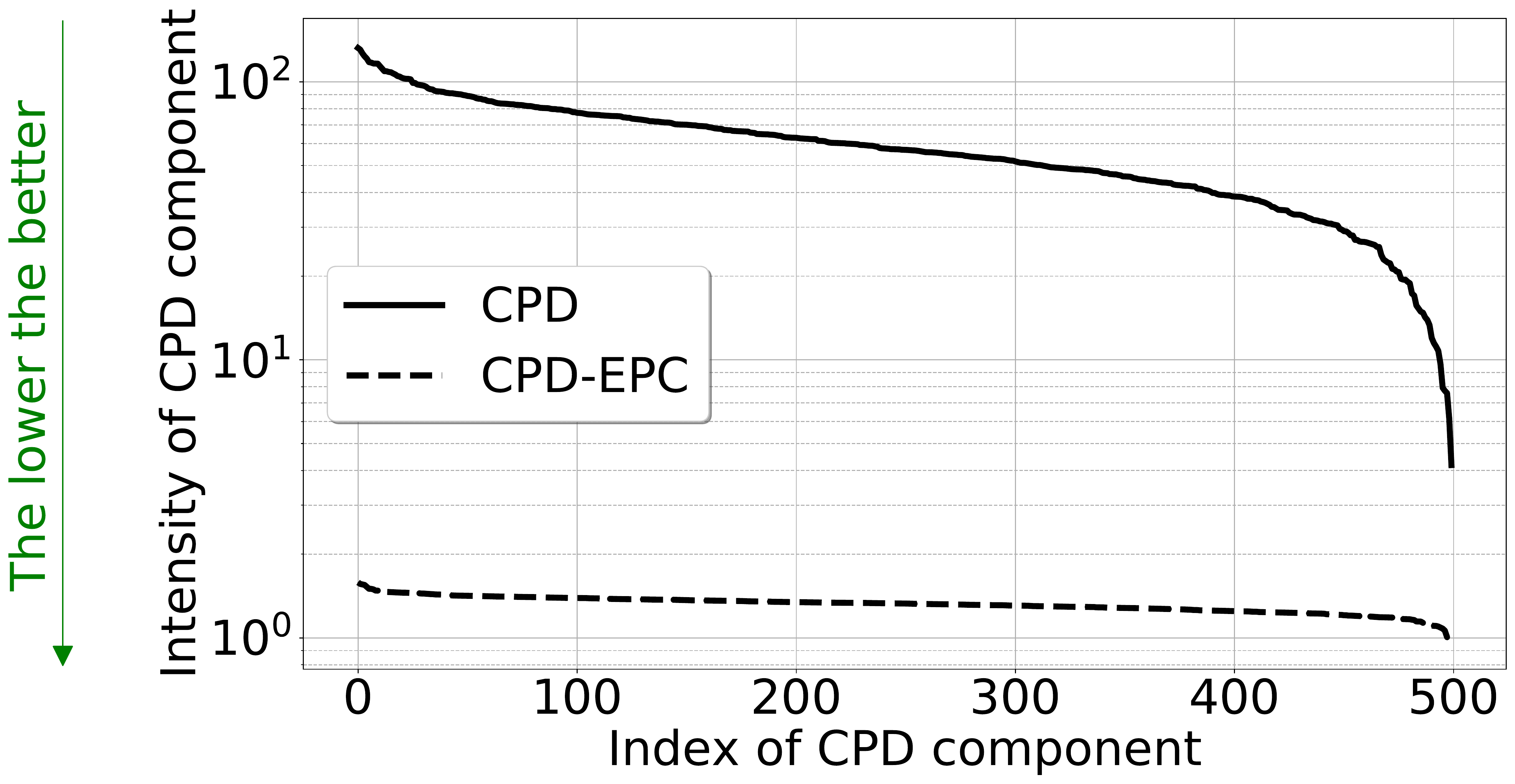}
\hfill
\includegraphics[width=0.48\linewidth]{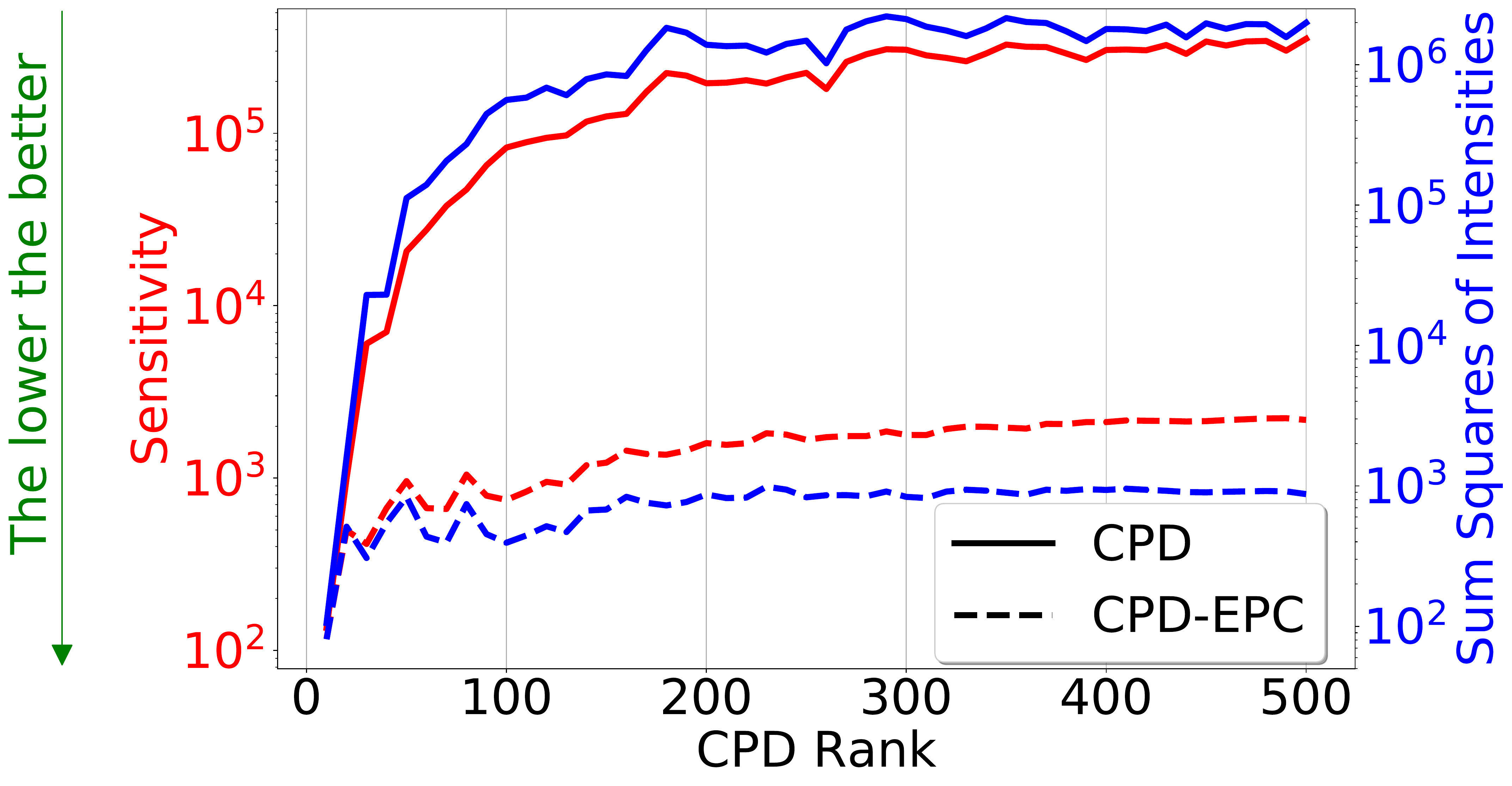}
\centering
\caption{(Left) Intensity (Frobenius norm) of rank-1 tensors in CPDs of the kernel in the $4^{th}$ layer of ResNet-18. (Right) Sum of  squares of the intensity and Sensitivity vs Rank of CPD. EPC-CPD demonstrates much lower intensity and sensitivity as compared to CPD.}
\label{fig:norm_of_lambdas}
\end{figure}

\subsection{Contributions}

In this paper, we address the problem of CNN stability compressed by CPD. 
The key advantages and major contributions of our paper are the following:
\begin{itemize}
    \item We propose a new stable and efficient method to perform neural network compression based on low-rank tensor decompositions.
    \item 
    %We deal with the degeneracy, the most severe problem in CPD of convolutional kernels. We propose a method to seek a CPD with minimal sensitivity and intensity.    
    We demonstrate how to deal with  the degeneracy, the most severe problem when approximating convolutional kernels with CPD.  Our approach allows finding CPD a reliable representation with minimal sensitivity and intensity.
    
    \item 
    %We propose a  two-stages kernel compression (Figure~\ref{fig:tucker2epc}). First, the Tucker-2 (TKD-2) decomposition is used to seek the best subspace for the Input and Output channels. Further compression of the Core tensor in TKD-2 is achieved by CPD. This two-stages procedure replaces one convolutional layer by a sequence of 5 smaller convolutional layers. 
    We show that the combination of Tucker-2 (TKD) and the proposed stable CPD (Fig.~\ref{fig:tucker2epc}) outperforms CPD in terms of accuracy/compression trade-off.
    
    % \item We propose a novel method to find the optimal multi-linear rank for the TKD-2.
    
\end{itemize}
We provide results of extensive experiments to confirm the efficiency of the proposed algorithms. Particularly, we empirically show that the neural network with weights in factorized CP format obtained using our algorithms is more stable during fine-tuning and recovers faster (close) to initial accuracy.

\section{Stable Tensor Decomposition Method}\label{sec::method_idea}

\subsection{CP Decomposition of Convolutional Kernel}

In CNNs, the convolutional layer performs mapping of an input (source) tensor $\tX$ of size $H \times W \times S$ into output (target) tensor $\tY$ of size $H' \times W' \times T$ following the relation 
\begin{equation} \label{eq1}
\displaystyle \tY_{h',w',t} = \sum_{i=1}^D \sum_{j=1}^D \sum_{s=1}^S {\tilde{\tK}}_{i,j,s,t} \mathcal{X}_{h_i,w_j,s},
%\notag
\end{equation}
%,
where $h_i = (h' - 1)$ $\Delta + i - P \text{, and } w_j = (w' - 1)$ $\Delta + j - P$, and $\tilde{\tK}$ is an order-4 kernel tensor of size $D \times D \times S \times T$, $\Delta$ is stride, and $P$ is zero-padding size.
 
Our aim is to decompose the kernel tensor $\tilde{\tK}$ by the CPD or the TKD. 
As it was mentioned earlier, we treat the kernel $\tilde{\tK}$ as order-3 tensor $\tK$ of the size $D^2 \times S \times T$, and represent the kernel $\tK$
by sum of $R$ rank-1 tensors \\[-1em]
\be \label{eq3}
\begin{split}
\tK \simeq \hat{\tK} = \sum_{r=1}^R {\ba}_r \circ {\bb}_r \circ  {\bc}_r, 
\end{split}
\ee
where $\bA = [\ba_1, \ldots, \ba_R]$, $\bB = [\bb_1, \ldots, \bb_R]$ and $\bC = [\bc_1, \ldots, \bc_R]$ are factor matrices of size $D^2 \times R$, $S\times R$ and $T\times R$, respectively. See an illustration of the model in Fig.~\ref{fig:tucker2cp}. The tensor $ \hat{\tK} = \llbracket \bA, \bB, \bC \rrbracket$ in the Kruskal format uses $(D^2 + S + T) \times R$ parameters.

\subsection{Degeneracy and its effect to CNN stability}

Degeneracy occurs in most CPD of the convolutional kernels. The Error Preserving Correction (EPC) method \cite{Phan2019} suggests a correction to the decomposition results in order to get a
more stable decomposition with lower sensitivity.
There are two possible measures for assessment of the degeneracy degree of the CPD: sum of Frobenius norms of the rank-1 tensors\cite{Phan2019}\\[-2em]
%\begin{align}
\be
{\tt{sn}}(\llbracket \bA,\bB,\bC \rrbracket) = \sum_{r = 1}^{R} \|\ba_r \circ \bb_r \circ \bc_r\|_F^2 
\ee
%\end{align}
and sensitivity, defined as follows.

\begin{definition}[Sensitivity \cite{Petr_sensitivity}]\label{def:sensitivity}
Given a tensor $\tT = \llbracket \bA, \bB, \bC \rrbracket$, define the sensitivity as
\be
 {\tt{ss}}(\llbracket \bA,\bB,\bC \rrbracket) = \lim_{\sigma^2 \rightarrow 0} \frac{1}{R\sigma^2} E\{\|\tT - \llbracket \bA + \delta \bA, \bB + \delta \bB, \bC + \delta \bC \rrbracket \|_F^2\}
 \label{eq_ss}
\ee 
where $\delta \bA$, $\delta \bB$, $\delta \bC$ have random i.i.d. elements from $N(0,\sigma^2).$
\end{definition}
The sensitivity of the decomposition can be measured by the expectation ($E\{\cdot\}$)  of the normalized squared Frobenius norm of the difference. In other words, sensitivity of the tensor $\tT = \llbracket \bA, \bB, \bC \rrbracket$ is a measure with respect to perturbations in individual factor matrices.
CPDs with high sensitivity are usually useless.

\begin{lemma}
\be
{\tt{ss}}(\llbracket \bA,\bB,\bC \rrbracket) = K\tr\{(\bA^T \bA) \* (\bB^T \bB)\} + I \tr\{(\bB^T \bB) \* (\bC^T \bC) \} + J \tr\{ (\bA^T \bA) \* (\bC^T \bC) \}. \label{eq_ss_full}
\ee

where $\*$ denotes the Hadamard element-wise product.
\end{lemma}
\begin{proof}
First, the perturbed tensor in (\ref{eq_ss}) can be expressed as sum of 8 Kruskal terms 
\begin{align*}
&\llbracket \bA + \delta\bA, \bB + \delta\bB,  \bC + \delta\bC \rrbracket = \llbracket \bA, \bB,  \bC \rrbracket + \llbracket \delta\bA, \bB,  \bC \rrbracket
+ \llbracket \bA, \delta\bB, \bC \rrbracket
+ \llbracket \bA, \bB, \delta\bC \rrbracket \notag \\
& \qquad + \llbracket \delta\bA, \delta\bB, \bC \rrbracket 
+ \llbracket \delta\bA, \bB, \delta\bC \rrbracket 
+ \llbracket \bA, \delta\bB, \delta\bC \rrbracket 
+ \llbracket \delta\bA, \delta\bB, \delta\bC \rrbracket .
\end{align*}
Since these Kruskal terms are uncorrelated and expectation of the terms composed by two or three factor matrices $\delta\bA$, $\delta\bB$ and $\delta\bC$ are negligible, the expectation in  (\ref{eq_ss}) can be expressed in the form 
\begin{align}
E\{\|\tT - \llbracket \bA + \delta\bA, \bB + \delta\bB,  \bC + \delta\bC \rrbracket\|_F^2\} 
&= 
E\{\|\llbracket \delta\bA, \bB,  \bC \rrbracket\|_F^2\} + 
\notag \\
& 
+ E\{\|\llbracket \bA, \delta\bB, \bC \rrbracket\|_F^2\}
+ E\{\|\llbracket \bA, \bB, \delta\bC \rrbracket\|_F^2\} \label{eq_ss2}\, .
\end{align}
Next we expand the Frobenius norm of the three Kruskal tensors
\be
E\{\|\llbracket \delta\bA, \bB,  \bC \rrbracket\|_F^2\} &=& E\{\|\left((\bC \odot \bB) \otimes \bI\right) \vtr{\delta\bA}\|^2\} \notag \\
&=& E\{\tr(\left(\bC \odot \bB) \otimes \bI\right)^T \left((\bC \odot \bB) \otimes \bI\right) \vtr{\delta\bA}\vtr{\delta\bA}^T)\} \notag \\
&=& 
\sigma^2 \tr((\bC \odot \bB)^T (\bC \odot \bB) \otimes \bI) \notag \\
&=& 
R\sigma^2 \tr((\bC^T \bC) \* (\bB^T \bB))   \\
E\{\|\llbracket \bA, \delta \bB, \bC \rrbracket\|_F^2\} &=& R\sigma^2 \tr((\bC^T \bC) \* (\bA^T \bA)) \\
E\{\|\llbracket \bA, \bB,  \delta \bC \rrbracket\|_F^2\} &=& 
R\sigma^2 \tr((\bB^T \bB) \* (\bA^T \bA))
\ee
where $\odot$ and $\otimes$ are Khatri-Rao and Kronecker products, respectively. 

Finally, we replace these above expressions into (\ref{eq_ss2}) to obtain the compact expression of sensitivity.
\end{proof}

\subsection{Stabilization Method}

\subsubsection{Sensitivity minimization}

The first method to correct CPD with diverging components proposed in  \cite{Phan2019} minimizes the sum of Frobenius norms of rank-1 tensors while the approximation error is bounded. 
In \cite{Petr_sensitivity}. the Krylov Levenberg-Marquardt algorithm was proposed for the CPD with bounded sensitivity constraint.  

In this paper, we propose a variant of the EPC method which minimizes the sensitivity of the decomposition while preserving the approximation error, i.e., 
\be
    \min_{\{\bA, \bB, \bC\}} &\quad {\tt{ss}}(\llbracket {\mathbf{A}}, {\mathbf{B}}, {\mathbf{C}} \rrbracket) \label{eq_min_ss}\\
    \text{s.t.} &\quad \| \tK - \llbracket {\mathbf{A}}, {\mathbf{B}}, {\mathbf{C}} \rrbracket\|_F^2 \le \delta^2 \notag\,.
\ee
The bound, $\delta^2$, can represent the approximation error of the decomposition with diverging components.  
%
% \begin{align}
% \min \quad & f(\boldsymbol{\theta}) = ||\boldsymbol{\eta}||_2^2 = \sum_{r=1}^R \eta_r^2 \\
% \textrm{s.t.} \quad & c(\boldsymbol{\theta}) = ||X - \hat{X}(\boldsymbol{\theta})|| \leq \delta^2, 
% \end{align}
%
% \noindent where $\boldsymbol{\theta}$ represents a vector of all model parameters and $\hat{X}(\boldsymbol{\theta})$ represents the estimated tensor of $X$ constructed from $\boldsymbol{\theta}$. 
Continuing the CPD using a new tensor $\hat{\tK}=\llbracket {\mathbf{A}}, {\mathbf{B}}, {\mathbf{C}} \rrbracket$ with a lower sensitivity can improve its convergence. 

\subsubsection{Update rules}

We derive alternating update formulas for the above optimization problem.
While $\bB$ and $\bC$ are kept fixed, the objective function is rewritten to update $\bA$ as 
\be
\min_{\bA}  && \quad \tr\{(\bA^T \bA) \* \bW\}  =  \|\bA \diag(\bw) \|_F^2 \label{eq_prob_A_ss}\\
\text{s.t.} && \quad \|\bK_{(1)}  - \bA \bZ^T \|_F^2 \le \delta^2, \notag 
\ee
where $\bK_{(1)}$ is mode-1 unfolding of the kernel tensor $\tK$, $\bZ = \bC \odot \bB$
and $\bW = \bB^T\bB + \bC^T \bC$ is a symmetric matrix of size $R\times R$, $\bw = [\sqrt{w_{1,1}}, \ldots, \sqrt{w_{R,R}}]$ is a vector of length $R$ taken from the diagonal of $\bW$.

\begin{remark}
The problem (\ref{eq_prob_A_ss}) can be reformulated as a regression problem with bound constraint 
\be
\min_{\widetilde{\bA}} && \quad \|\widetilde{\bA} \|_F^2 \label{eq_prob_A_ss2}\\
\text{s.t.} && \quad \|\bK_{(1)}  -  \widetilde{\bA}  \widetilde{\bZ}^ T \|_F^2 \le \delta^2, \notag 
\ee
where $\widetilde{\bA} = \bA \diag(\bw)$ and $\widetilde{\bZ} = \bZ \diag(\bw^{-1})$.
This problem can be solved in closed form solution through the quadratic programming over a sphere \cite{Phan:2020:10.1007/s00521-019-04191-z}. We skip the algorithm details and refer to the solver in \cite{Phan:2020:10.1007/s00521-019-04191-z}.
\end{remark}

\begin{remark} 
If factor matrices $\bB$ and $\bC$ are normalized to unit length columns, i.e., $\|\bb_r\|_2 = \|\bc_r\|_2 = 1$, $r = 1, \ldots, R$, then all entries of the diagonal of $\bW$ are identical.
The optimization problem in (\ref{eq_prob_A_ss}) becomes seeking a weight matrix, $\bA$, with minimal norm
\be
\min_{\bA} && \quad  \|\bA\|_F^2 \label{eq_prob_A_epc}\\
\text{s.t.} && \quad \|\bK_{(1)}  - \bA \bZ^T \|_F^2 \le \delta^2 \notag.
\ee
This sub-optimization problem is similar to that in the EPC approach \cite{Phan2019}.
\end{remark}

\subsection{Tucker Decomposition with Bound Constraint}

Another well-known representation of multi-way data is the Tucker Decomposition \cite{tucker1963implications}, which decomposes a given tensor into a core tensor and a set of factor matrices (see Fig.~\ref{fig:tucker2} for illustration). The Tucker decomposition is particularly suited as prior-compression for CPD. In this case, we compute CPD of the core tensor in TKD,  which is of smaller dimensions than the original kernels.

For our  problem, we are interested in the Tucker-2 model (see Fig.~\ref{fig:tucker2})
\begin{equation}
    \tK \simeq \tG \times_{2} \bU \times_{3} \bV, 
\end{equation}
where $\tG$ is the core tensor of size $D^2 \times R_{1} \times R_2$, $\bU$ and $\bV$ are matrices of size $S\times R_1$ and $T\times R_2$, respectively.
Because of rotational ambiguity, without loss in generality, the matrices $\bU$ and $\bV$ can be assumed to have orthonormal columns.

Different from the ordinary TK-2, we seek the smallest TK-2 model which holds the approximation error bound $\delta^2$\cite{Phan_TT_part1u}, i.e., 
\be
\min_{\{\tG, \bU, \bV\}} & \quad R_1 S + R_2 T + R_1 R_2 D^2 \label{eq_tk2_bound}\\
\text{s.t.} & \quad \|\tK - \tG \times_{2} \bU \times_{3} \bV \|_F^2 \le \delta^2 \notag \\
& \quad \bU^T \bU = \bI_{R_1}, \bV^T \bV = \bI_{R_2}\,. \notag
\ee

We will show that the core tensor $\tG$ has closed-form expression as in the HOOI algorithm for the orthogonal Tucker decomposition \cite{HOOI:Lathauwer:2000}, and the two-factor matrices, $\bU$ and $\bV$, can be {\emph{sequentially estimated}} through Eigenvalue decomposition (EVD).
%
%
%\begin{lemma} ${\tX}_2^{\star} = \bX_1^T \bullet \tY \bullet \bX_3^T$ is an optimal core tensor of the optimization problem (\ref{eq_tk2_bound}).
%\end{lemma}
\begin{lemma} The core tensor ${\tG}$ has closed-form expression ${\tG}^{\star} = \tK \times_2 \bU^T \times_3 \bV^T$.
\end{lemma}
\begin{proof} From the error bound condition, we can derive 
\be
%\varepsilon^2 \ge \|\tY - \bX_1 \bullet \tX_2 \bullet \bX_3\|_F^2 =  \|\tY\|_F^2 - \|{\tX}_2^{\star}\|_F^2 +  \|\tX_2 - \tX_2^{\star}\|_F^2 \notag
\delta^2 &\ge& \|\tK -  \tG \times_{2} \bU \times_{3} \bV\|_F^2
% &=& \|\tK\|_F^2 + \| \tG \times_{2} \bU \times_{3} \bV\|_F^2 -2 \langle \tK,  \tG \times_{2} \bU \times_{3} \bV \rangle \notag \\
% &=&  \|\tK\|_F^2 + \|\tG\|_F^2 -2 \langle \tK \times_2 \bU^T \times_3 \bV^T, \tG \rangle \notag \\
% &=&  \|\tK\|_F^2 + \|\tG\|_F^2 -2 \langle \tG^{\star}, \tG \rangle \notag \\
=  \|\tK\|_F^2 - \|{\tG}^{\star}\|_F^2 +  \|\tG - \tG^{\star}\|_F^2,  \notag
\ee
which indicates that the core tensor can be expressed as  
$\tG = \tG^{\star} + \tE$,
 where $\tE$ is an error tensor such that its norm $\gamma^2 = \|\tE\|_F^2 \le \delta^2 + \|\tG^{\star}\|_F^2 - \|\tK\|_F^2$.

Next define a matrix $\bQ_1$   of size $S \times S$  
\be
\bQ_1(i,j) &=&  \sum_{r = 1}^{R_2}  \bV{(:,r)}^T  \, \bK(:,i,:) \bK(:,j,:)^T  \,  \bV{(:,r)} \label{eq_Q1} \,.
\ee
Assume that $\bV^{\star}$ is the optimal factor matrix with the minimal rank $R_2^{\star}$.  The optimization in  (\ref{eq_tk2_bound}) becomes the rank minimization problem for $\bU$
\be
\min_{\bU} &&\quad {\text{rank}}(\bU)  \quad \label{eq_probU} \\
\text{s.t.}&&\quad \tr(\bU^T \bQ_1 \bU) \ge \|\tK\|_F^2 + \gamma^2  - \delta^2\,, \notag \\
&& \quad \bU^T \bU = \bI_{R_1}\notag.
\ee
The optimal factor matrix $\bU^{\star}$ comprises $R_1$ principal eigenvectors of $\bQ_1$, where $R_1$ is the smallest number of eigenvalues, $\lambda_1 \ge \lambda_2 \ge \cdots \ge \lambda_{R_1}$ such that their norm exceeds the bound $\|\tY\|_F^2 - \delta^2 + \gamma^2$, 
that is, $\displaystyle
\sum_{r = 1}^{R_1} \lambda_r \ge \|\tK\|_F^2 - \delta^2 + \gamma^2  > \sum_{r = 1}^{R_1-1} \lambda_r
$.
It is obvious that the minimal number of columns $R_1$ is achieved, when the bound $\|\tK\|_F^2 + \gamma^2  - \delta^2$ is smallest, i.e., $\gamma = 0$.
Implying that the optimal $\tG$ is $\tG^{\star}$. This completes the proof.
\end{proof}

Similar to the update of $\bU$, the matrix $\bV$ comprises $R_2$ principal eigenvectors of the matrix $\bQ_2$ of size $T \times T$
\be
\bQ_2(i,j) &=&  \sum_{r = 1}^{R_1}   \bU(:,r)^T  \, \bK(:,:,i)^T \bK(:,:,k)\, \bU(:,r)\label{eq_Q2},
\ee
where $R_2$ is either given or determined based on the bound $\|\tY\|_F^2 - \delta^2$.
%The algorithm is summarized in Algorithm~\ref{alg_TT3}, and  sequentially updates  $\bX_1$ and $\bX_3$. %, and only needs to initialize $\bX_3$ as an identity matrix.
The algorithm for TKD sequentially updates  $\bU$ and $\bV$.

\section{Implementation} \label{sec::method}

% \subsection{Factorized Layers}

Our method for neural network compression includes the following main steps (see Fig.~\ref{fig:factorized_layers}):
\begin{enumerate}
\item 
%Convolutional  kernels are first decomposed into two factors and a core tensor using the Tucker-2 decomposition (TKD).
%The model is determined to best represent the kernel at a bounded error.
Each convolutional kernel is approximated by a tensor decomposition (CPD/TKD-CPD in case of ordinary convolutions and SVD in case of $1 \times 1$ convolution) with given rank R.

\item 
The CP decomposition with diverging components is corrected using the error preserving method. The result is a new CP model with minimal sensitivity.
%To tackle the problem of diverging components in CPD decomposition, is corrected using the error preserving (EPC) method. As a result, for the core tensor we obtain a new CPD with minimal sensitivity.

\item 
%A new convolutional layer with parameters in the TK-Kruskal format is built up then refined in a network of 5 layers. 
An initial convolutional kernel is replaced with a tensor in CPD/TKD-CPD or SVD format, which is equivalent to replacing one convolutional layer with a sequence of convolutional layers with a smaller total number of parameters. 

% \item Stages (1) - (3) repeated until all convolutional layers are processed.

\item The entire network is then fine-tuned using backpropagation.
\end{enumerate}

{\bf CPD Block} results in three convolutional layers with shapes $(C_{in} \times R \times 1 \times 1)$, depthwise $(R \times R \times D \times D)$  and $(R \times C_{out}  \times 1 \times 1)$, respectively (see Fig.~\ref{fig:cpd-layer}).
% 
% % Paragraph describing how CPD block is obtained
% Our method uses CP-3 decomposition. In order to apply it to 4-dimensional convolutional kernel, we reshape it from $D \times D \times I \times O$ to $D^2 \times I \times O$ shape. After applying CP-3 decomposition with rank $R$ to reshaped kernel, we obtain 3 factor matrices with shapes $I \times R, D^2 \times R \text{ and } R \times O$. Further, these three factor matrices are reshaped to 4-d tensors with shapes $(1 \times 1 \times I \times R), (D \times D \times R \times 1) \text {and } (1 \times 1 \times R \times O)$. These three tensors are further used as kernels for sequential convolutional layers in CPD block \ref{fig:factorized_layers}(top-left).
% 
% % Paragraph describing efficiency of CPD block
In obtained structure, all spatial convolutions are performed by central $D \times D $ group convolution with $ R $ channels. $1 \times 1$ convolutions allow the transfer of input data to a more compact channel space (with $ R $ channels) and then return data to initial channel space.
% In obtained structure, all spatial convolutions are performed by central $D \times D$ group convolution with $R$ channels. $1 \times 1$ convolutions play the role of transferring information between channel dimensions of data and transfer data to a more compact channel space $R$ and vice versa. 
% In practice, decomposition rank $R$ is sufficiently smaller than number of input ($I$) and output ($O$) channels. As a result, convolutional layer with $D^2IO$ parameters is replaced by compact structure with $R(D^2+I+O)$ parameters, which results in reduced memory consumption and CNN acceleration.

{\bf TKD-CPD Block}
 is similar to the CPD block, but has 4 $(1 \times 1)$ convolutional layers with the condition that the CP rank must exceed the multilinear ranks, $R_1$ and $R_2$ (see Fig.~\ref{fig:tkd-cpd-layer}).
 This structure allows additionally to reduce the number of parameters and floating operations in a factorized layer. Otherwise, when $R < R_1$ and $R < R_2$, sequential $1 \times 1$ convolutions can be merged into one $1 \times 1$ convolution, converting the TKD-CPD layer format to CPD block.
 
% TKD-CPD block has a similar structure to the CPD block. The only difference is that it has 4 $1 \times 1$. This structure allows additionally to reduce the number of parameters and floating operations in a factorized layer in a case $R \ge R_1$ and $R \ge R_2$. Otherwise, sequential $1 \times 1$ convolutions can be merged into one $1 \times 1$ convolution, converting the TKD-CPD layer format to CPD block.

{\bf SVD Block} is a variant of CPD Block but comprises only two-factor layers, computed using SVD. Degeneracy is not considered in this block, and no correction is applied (see Fig.~\ref{fig:svd-layer}). 

%Paragraph, that states that modern CNN have a huge 1x1 convs and describes why do CNNs need them
% Modern CNN architectures extensively use $1 \times 1$ convolutional layers for efficient network design. Basically, they play a role of channel dimension reduction of the data before more heavy spatial convolutions ($3 \times 3$, $5 \times 5$, $7 \times 7$) convolutions and recovering after, which sufficiently helps to reduce number of parameters in network without loss of quality \cite{Szegedy2015GoingDW, HeZRS15, Iandola2017SqueezeNetAA, howard2019searching, tan2019efficientnet}.

%Paragraph, that describes why we can not use CPD in 2d case and that SVD is special case of CPD and describes structure of SVD layer
% Kernel of $1 \times 1$ convolutions can be represented as $I \times O$ matrix, where $I$ is number of input channels and $O$ is number of output channels. In case of  2 dimensions CPD decomposition turns into well-known SVD \cite{cichocki2016tensor}, that decomposes big matrix $I \times O$ into two smaller ones $I \times R$ and $R \times O$, where $R$ is decomposition rank. Obtained factor matrices are used as kernels for sequential $1 \times 1$ convolutions in SVD block\ref{fig:factorized_layers}(top-right).

\begin{figure}[t]
    \centering
    \begin{subfigure}{.54\linewidth}
        \centering
        \includegraphics[width=\linewidth]{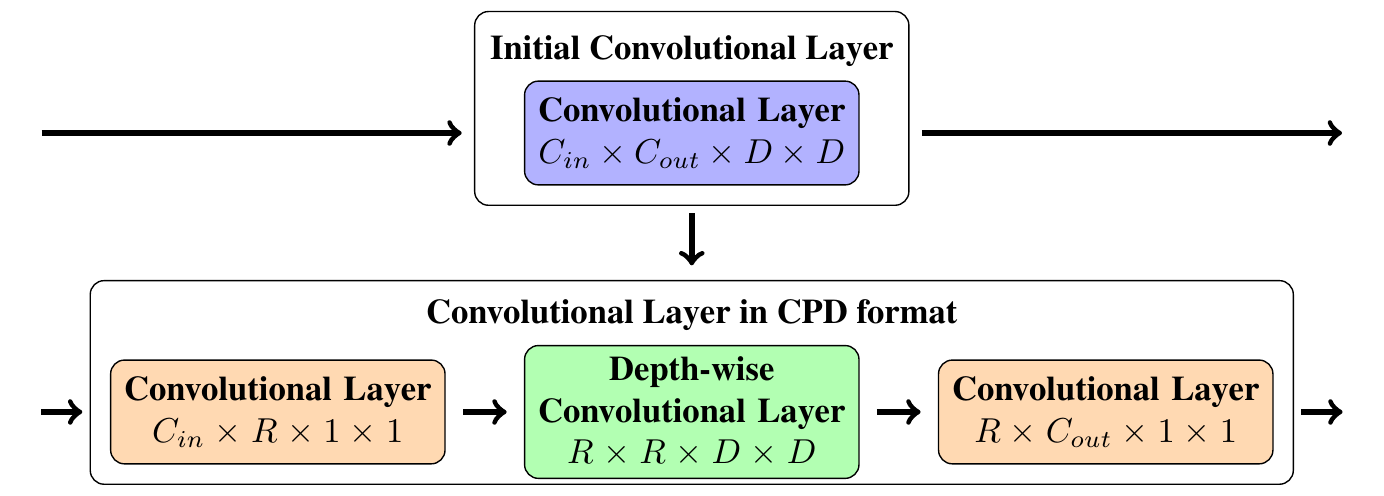}
        \caption{}\label{fig:cpd-layer}
    \end{subfigure}
    \begin{subfigure}{.4\linewidth}
        \centering
        \includegraphics[width=\linewidth]{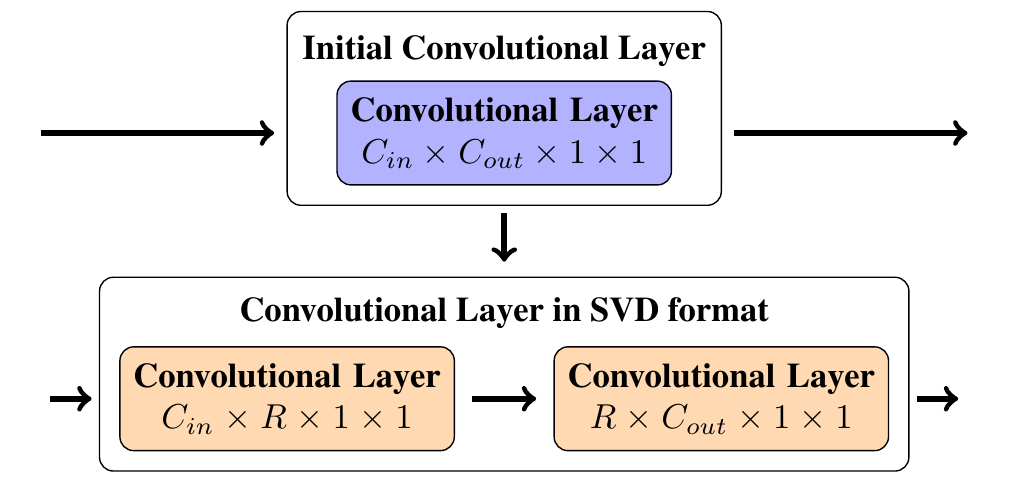}
        \caption{}\label{fig:svd-layer}
    \end{subfigure}
    \begin{subfigure}{.87\linewidth}
        \centering
        \includegraphics[width=\linewidth]{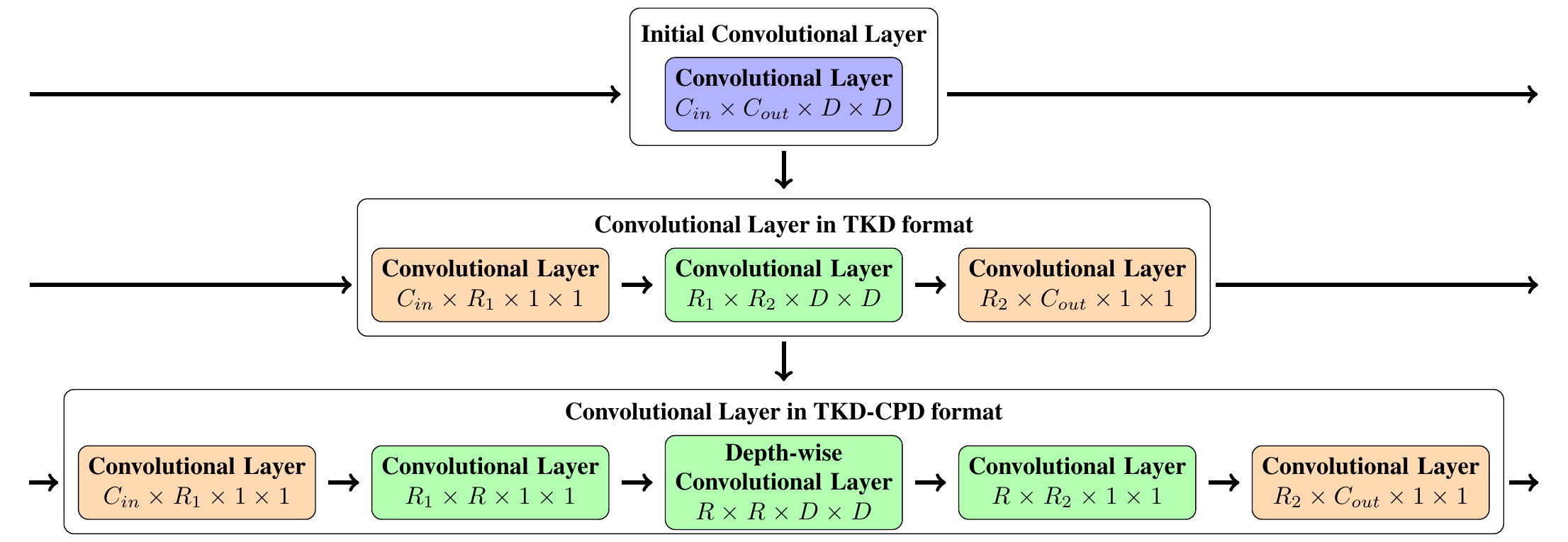}
        \caption{}\label{fig:tkd-cpd-layer}
    \end{subfigure}
   \caption{Graphical illustration to the proposed layer formats that show how decomposed factors are used as new weights of the compressed layer. $C_{in}, C_{out}$ are the number of input of and output channels and $D$ is a kernel size. (a) CPD layer format, $R$ is a CPD rank. (b)  SVD layer format, $R$ is a SVD rank. (c) TKD-CPD layer format, $R$ is a CPD rank, $R_1$ and $R_2$ are TKD ranks.}
    \label{fig:factorized_layers}
\end{figure}

{\bf Rank Search Procedure.} Determination of CP rank is an NP-hard problem \cite{hillar2013most}.  
We observe that the drop in accuracy by a factorized layer influences accuracy with fine-tuning of the whole network.
In our experiments, we apply a heuristic binary search to find the smallest rank such that drop after single layer fine-tuning does not exceed a predefined accuracy drop threshold $EPS$.

\section{Experiments}

We test our algorithms on three representative convolutional neural network architectures for image classification: \textit{VGG-16} \cite{SimonyanZ14a}, \textit{ResNet-18}, \textit{ResNet-50} \cite{HeZRS15}. We compressed $7 \times 7$ and $3 \times 3$ convolutional kernels with CPD, CPD with sensitivity correction (CPD-EPC), and Tucker-CPD with the correction (TKD-CPD-EPC).
The networks after fine-tuning are evaluated through  \textit{top 1} and \textit{top 5} accuracy on  \textit{ILSVRC-12} \cite{imagenet_cvpr09} and \textit{CIFAR-100} \cite{cifar09}. 

We conducted a series of layer-wise compression experiments and measured accuracy recovery and whole model compression of the decomposed architectures. Most of our experiments were devoted to the approximation of single layers when other layers remained intact. In addition, we performed compression of entire networks.
%in order to evaluate the resulting accuracy for the whole network with a high compression ratio.

The experiments were conducted with the popular neural networks framework \textit{Pytorch} on GPU server with NVIDIA V-100 GPUs. As a baseline for ILSVRC-12 we used a pre-trained model shipped with \textit{Torchvision}. Baseline CIFAR-100 model was trained using the Cutout method. The fine-tuning process consists of two parts: local or single layer fine-tuning, and entire network fine-tuning. The model was trained with an SGD optimizer with an initial learning rate of $10^{-3}$ and learning decay of $0.1$ at each loss saturation stage, weight decay was set as $10^{-4}$.

\subsection{Layer-Wise Study}

\begin{figure}[t!]
\centering
\begin{subfigure}{.48\linewidth}
\includegraphics[width=\linewidth,height = .6\linewidth]{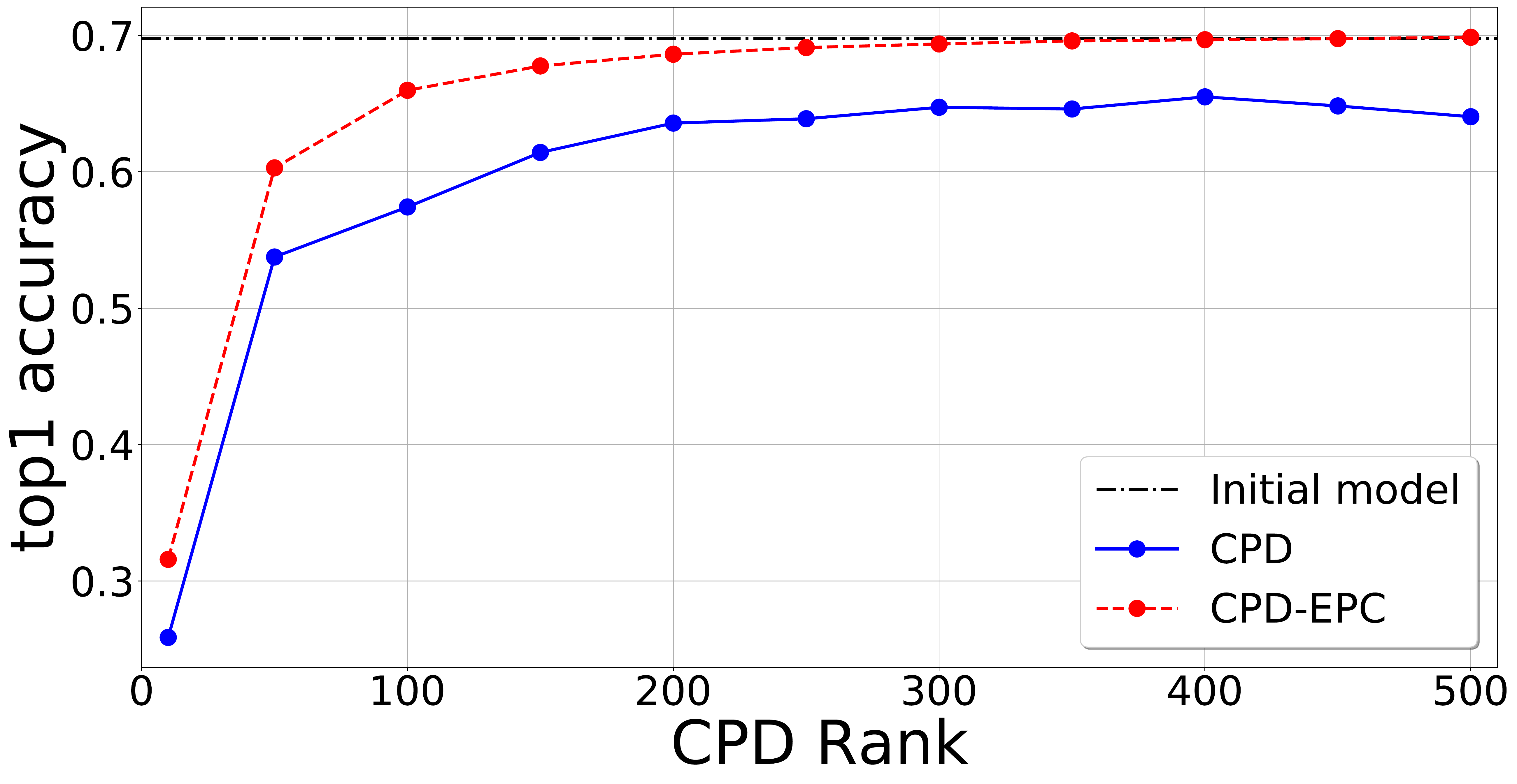}
\end{subfigure}
\hfill
\begin{subfigure}{.48\linewidth}
\includegraphics[width=\linewidth,height = .6\linewidth]{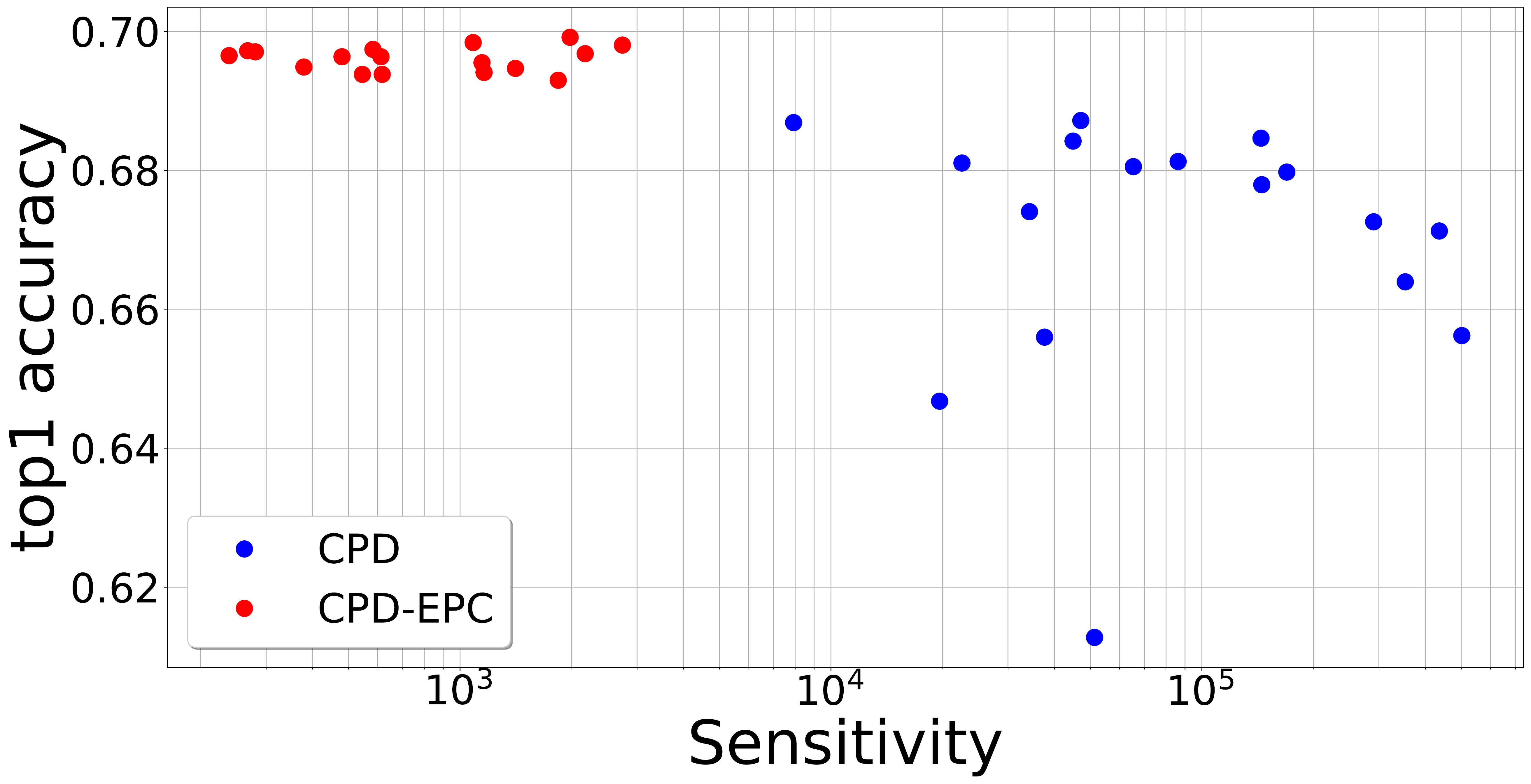}\label{fig:single_layers_all_ranks_2}
\end{subfigure}

\caption{(Left) Performance evaluation of ResNet-18  on ILSVRC-12 dataset after replacing {\tt{layer4.1.conv1}} by its approximation using CPD and CPD-EPC with various ranks. The networks are fine-tuned after compression. (Right) Top-1 accuracy and sensitivity of the models estimated using CPD (blue) and CPD-EPC (red). Each model has a single decomposed layer with the best CP rank and was fine-tuned after compression. CPD-EPC outperforms CPD in terms of accuracy/sensitivity trade-off.
}
\label{fig:single_layers_all_ranks}
\end{figure}
\footnotetext{{\tt{layer4.1.conv1}} -- layer 4, residual block 2(indexing starts with 0),  convolutional layer 1}

\subsubsection{CPD-EPC vs CPD}

For this study, we decomposed the kernel filters in 17 convolutional layers of ResNet-18 with different CP ranks, $R$, ranging from small (10) to relatively high rank (500). 
The CPDs were run with a sufficiently large number of iterations so that all models converged or there was no significant improvement in approximation errors. 

Experiment results show that for all decomposition ranks, the CPD-EPC regularly results in considerably higher \textit{top 1} and \textit{top 5} model accuracy than the standard CPD algorithm. 
Fig.~\ref{fig:single_layers_all_ranks}~(left) demonstrates an illustrative example for {\tt{layer4.1.conv1}}. An important observation is that the compressed network using CPD even with the rank of 500 (and fine-tuning) does not achieve the original network's accuracy. However, with EPC, the performances are much better and attain the original accuracy with the rank of 450. Even a much smaller model with the rank of 250 yields a relatively good result, with less than 1$\%$ loss of accuracy.  

Next, each convolutional layer in ResNet-18 was approximated with different CP ranks and fine-tuned. The best model in terms of top-1 accuracy was then selected. 
Fig.~\ref{fig:single_layers_all_ranks} (right) shows relation between the sensitivity and accuracy of the best models. %Table~\ref{table:single_layers_all_ranks}.
% The of models with single decomposed (using best rank). The networks are fine-tuned after compression.
It is straightforward to see that the models estimated using CPD exhibit high sensitivity, and are hard to train. The CPD-EPC  suppressed sensitivities of the estimated models and improved the performance of the compressed networks.
The CPD-EPC gained the most remarkable accuracy recovery on deeper layers of CNNs.

The effect is significant for some deep convolutional layers of the network with $\sim 2\%$ top-1 accuracy difference.

\subsubsection{CPD-EPC vs TKD-EPC}

Next, we investigated the proposed compression approach based on the hybrid TKD-CPD model with sensitivity control. Similar experiments were conducted for the CIFAR-100 dataset. The TK multi-linear ranks $(R_1, R_2)$ were kept fixed, while the CP rank varied in a wide range. 

In Fig.~\ref{fig:layer2.0.conv2_flops}, we compare accuracy of the two considered compressed approaches 
applied to the layer {\tt{4.0.conv1}} in ResNet-18. For this case, CPD-EPC still demonstrated a good performance.
The obtained accuracy is very consistent, implying that the layer exhibits a low-rank structure.
The hybrid TKD-CPD yielded a rather low accuracy for small models, i.e., with small ranks, which are much worse than the CPD-based model with less or approximately the same number of parameters. However, the method quickly attained the original top-1 accuracy and even exceeded the top-5 accuracy when the $R_{CP} \ge 110$.

Comparison of accuracy vs. the number of FLOPs and parameters for the other layers is provided in Fig.~\ref{fig:cifar100vs}. Each dot in the figure represents (accuracy, no. FLOPs) for each model. The dots for the same layers are connected by dashed lines. 
Once again, TKD-EPC achieved higher \textit{top 1} and \textit{top 5} accuracy with a smaller number of parameters and FLOPs, compared to CPD-EPC.

\begin{figure}[t!]
\centering
\includegraphics[width=1\linewidth,height = .20\linewidth, clip = true]{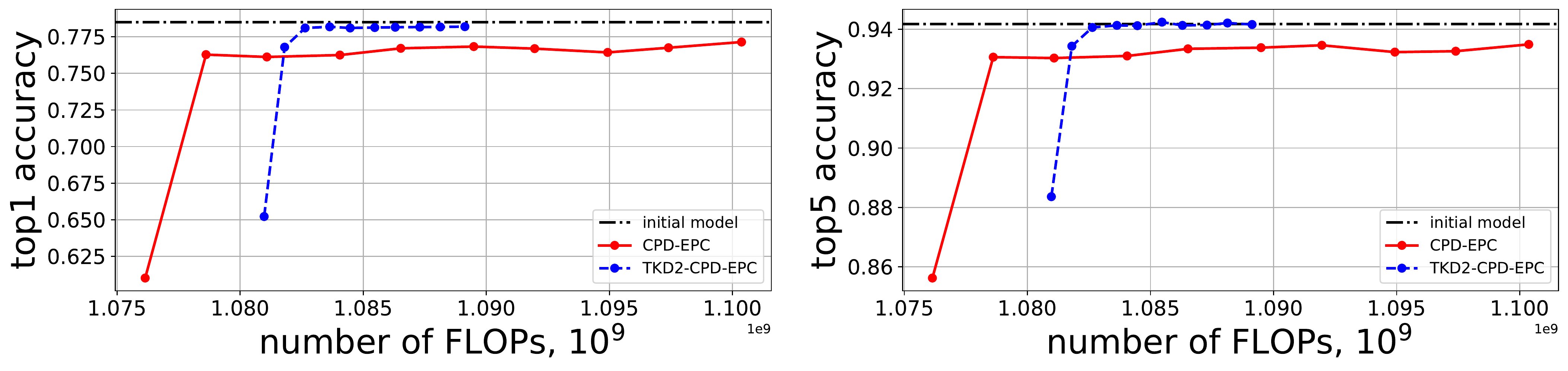} 
\caption{Performance comparison (top1 accuracy -- left, top5 accuracy -- right) of CPD-EPC and TKD-CPD-EPC in compression of the layer {\tt{4.0.conv1}} in the pre-trained ResNet-18 on ILSVRC-12 dataset. TKD-CPD-EPC shows better accuracy recovery with a relatively low number of FLOPs.  Initial model has $\approx 1.11 \times 10^9$ FLOPs.}
\label{fig:layer2.0.conv2_flops}
\end{figure}

\begin{figure}[t!]
\centering
\includegraphics[scale=0.10,width=1\linewidth]{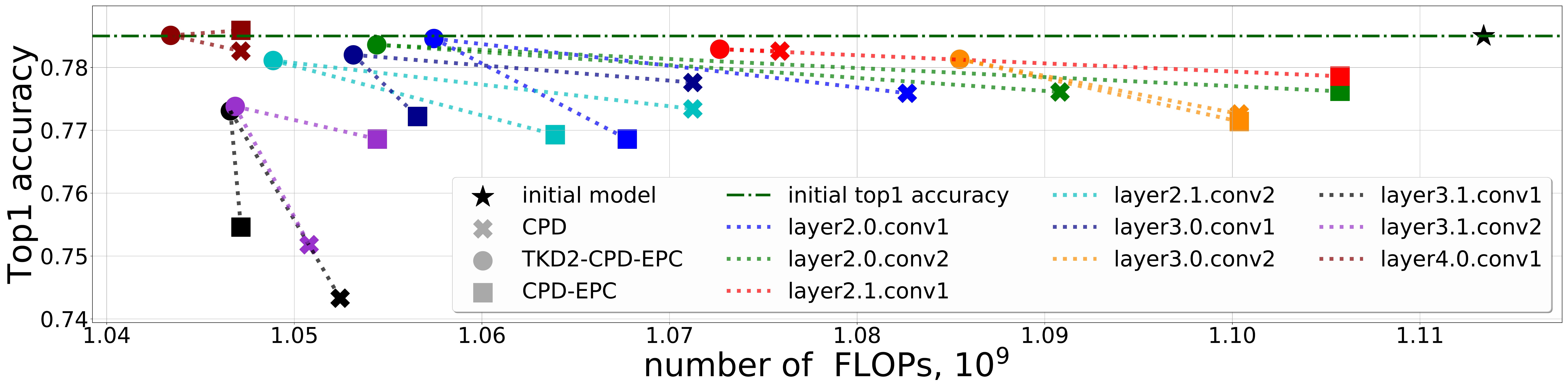}
\caption{Accuracy vs FLOPs for models obtained from ResNet-18 (CIFAR-100) via compression of one layer using standard CPD (cross), CPD-EPC (square), or TKD-CPD-EPC (circle) decomposition. Each color corresponds to one layer, which has been compressed using three different methods.
%Also, the black star marker represents an original ResNet-18 model. 
For each layer, {TKD-CPD-EPC} outperforms other decompositions in terms of FLOPs, or accuracy, or both.}
\label{fig:cifar100vs}
\end{figure}

\subsection{Full Model Compression}
% except for downsample $1 \times 1$ convolutions in ResNets
%\textbf{Image Classification.}
In this section, we demonstrate the efficiency of our proposed method in a full model compression of three well-known CNNs \textit{VGG-16} \cite{SimonyanZ14a}, \textit{ResNet-18}, \textit{ResNet-50} \cite{HeZRS15} for the ILSVRC-12. We compressed all convolutional layers remaining fully-connected layers intact. The proposed scheme gives $(\times 1.10, \times 5.26)$ for VGG-16, $(\times 3.82, \times 3.09)$ for ResNet-18 and $(\times 2.51, \times 2.64)$ for ResNet-50  reduction in the number of weights and FLOPs respectively. Table \ref{table:overall_results} shows that our approach yields a high compression ratio while having a moderate accuracy drop.

{\bf VGG \cite{SimonyanZ14a}.} We compared our method with other low-rank compression approaches on \textit{VGG-16}. The Asym method \cite{zhang15accelerating} is one of the first successful methods on the whole VGG-16 network compression. This method exploits matrix decomposition, which is based on SVD and is able to reduce the number of flops by a factor of 5.  
Kim et al. \cite{Kim2016} applied TKD with ranks selected by VBMF, and achieved a comparable compression ratio but with a smaller accuracy drop.
As can be seen from the table \ref{table:overall_results}, our approach outperformed both Asym and TKD in terms of compression ratio and accuracy drop.  

{\bf ResNet-18  \cite{HeZRS15}.} This architecture is one of the lightest in the ResNet family, which gives relatively high accuracy. Most convolutional layers in ResNet-18 are with kernel size $ 3 \times 3$, making it a perfect candidate for the low-rank based methods for compression. We have compared our results with channel pruning methods \cite{hua2018channelGating, zhuang2018discrimination, gao2018dynamic} and iterative low-rank approximation method \cite{gusak2019automated}. Among all the considered results, our approach has shown the best performance in terms of compression - accuracy drop trade-off.

{\bf ResNet-50 \cite{HeZRS15}.} Compared to \textit{ResNet-18}, \textit{ResNet-50} is a deeper and heavier neural network, which is used as backbone in various modern applications, such as object detection and segmentation. A large number of $1 \times 1$ convolutions deteriorate performance of low-rank decomposition-based methods. There is not much relevant literature available for compression of  this type of ResNet. 
To the best of our knowledge, the results we obtained can be considered the first attempt to compress the entire  \textit{ResNet-50}.

{\bf Inference time for Resnet-50.}
We briefly compare the inference time of Resnet-50 for the image classification task in Table~\ref{tab:resnet_inference}. The measures were taken on 3 platforms: CPU server with Intel\textregistered~Xeon\textregistered~Silver 4114 CPU 2.20 GHz, NVIDIA GPU server with \textregistered~Tesla\textregistered~V100 and Qualcomm mobile CPU \textregistered\hspace{1ex} Snapdragon\texttrademark\hspace{1ex}845. The batch size was choosen to yield small variance in inference measurements, e.g., 16 for the measures on CPU server, 128 for the GPU server and 1 for the mobile CPU.  

\begin{table}[!t]
{
\centering
\caption{Comparison of different model compression methods on ILSVRC-12 validation dataset. The baseline models are taken from {Torchvision}.}\label{table:overall_results}
\centering
%\resizebox{1\textwidth}{!}
{
%\begin{tabular*}{llccc}
\begin{tabular*}{1\linewidth}{@{\extracolsep{\fill}}llccc}
\multicolumn{1}{l}{\textbf{Model}} & \textbf{Method}                     & \textbf{$\downarrow$ FLOPs} & \textbf{$\Delta$ top-1} & \textbf{$\Delta$ top-5} \\ \hline
\multirow{3}{*}{VGG-16 }        &  Asym.  \cite{zhang15accelerating}     & $\approx5.00$                                    &   -    &  -1.00     \\ 
                                &  TKD+VBMF \cite{Kim2016}     &  4.93  &   -    &  -0.50     \\
                                 &  \textbf{Our} (EPS\footnotemark[1]=0.005)  &  \textbf{5.26} & \textbf{-0.92}      &  \textbf{-0.34}     \\
                                 \hline
\multirow{5}{*}{ResNet-18}       
% & {Network Slimming \cite{zhang2017slimming} by \cite{gao2018dynamic}}           & 1.39           & -1.77                                & -1.29        \\
                                %  & {Low-cost Col. Layers \cite{Dong2017lowCost}}       & 1.53           & -3.65                                & -2.30    \\
                                 & {Channel Gating NN \cite{hua2018channelGating}}          & 1.61           & -1.62                                & -1.03                               \\
                                %  & {Soft Filter Pruning  \cite{he2018soft}} & 1.72           & -3.18                                & -1.85                                \\
                                 & {Discrimination-aware Channel Pruning \cite{zhuang2018discrimination}}    & 1.89           & -2.29                                & -1.38\\
                                 & {FBS \cite{gao2018dynamic}}                        & 1.98           & -2.54                                & -1.46\\
                                 & {MUSCO \cite{gusak2019automated}}   & 2.42 & -0.47 & -0.30\\
                                 & \textbf{Our} (EPS\footnotemark[1]=0.00325)    & \textbf{3.09}  & \textbf{-0.69}                       & \textbf{-0.15}      \\ \hline
{ResNet-50}       & \textbf{Our} (EPS\footnotemark[1]=0.0028) & \textbf{2.64} & \textbf{-1.47}   &  \textbf{-0.71}    \\ \hline
\end{tabular*}}}
\raggedright
\\[1ex]
\minitab[p{.9\textwidth}]{\footnotemark[1] EPS: accuracy drop threshold. Rank of the decomposition is chosen to maintain the drop in accuracy lower than EPS.}
\end{table}

\begin{table}[!ht]
    \centering
    \caption{Inference time and acceleration for ResNet-50 on different platforms.}
    \label{tab:resnet_inference}
    \begin{tabular}{l c c}
        \multirow{2}{*}{Platform} & \multicolumn{2}{c}{Model inference time }\\
         &  Original & Compressed  \\\hline
        {Intel\textregistered~Xeon\textregistered  Silver 4114 CPU 2.20 GHz} & 3.92 $\pm$ 0.02 s & 2.84 $\pm$ 0.02 s \\
        {NVIDIA\textregistered Tesla\textregistered V100} & 102.3 $\pm$ 0.5 ms & {89.5 $\pm$ 0.2 ms } \\
        {Qualcomm\textregistered Snapdragon\texttrademark 845} & 221 $\pm$ 4 ms & {171 $\pm$ 4 ms } \\\hline
    \end{tabular}
\end{table}

\section{Discussion and Conclusions}\label{sec:discussion}

% Low-rank tensor approximations have been successfully applied to a wide range of fields, e.g., to retrieve the hidden latent variables \cite{phan2016tensor}, or to compress the data \cite{cichocki2016tensor}, or construct a system with low complexity in system identification \cite{batselier2016tensor, batselier2017tensor, favier2012nonlinear}.

Replacing a large dense kernel in a convolutional or fully-connected layer by its low-rank approximation is equivalent to substituting the initial layer with multiple ones, which in total have fewer parameters.
% When flattening and fully connected layers are replaced with tensor regression \cite{novikov2018exponential, kolbeinsson2019robust, cichocki2017tensor}, parameter savings can be obtained via tensor contraction layer \cite{kossaifi2018tensor}.
However, as far as we concerned, the sensitivity of the tensor-based models has never been considered before.
The closest method proposes to add regularizer on the Frobenius norm of each weight to prevent over-fitting.

In this paper, we have shown a more direct way to control the tensor-based network's sensitivity.
Through all the experiments for both ILSVRC-12 and CIFAR-100 dataset, we have demonstrated the validity and reliability of our proposed method for compression of CNNs, which includes a stable decomposition method with minimal sensitivity for both CPD and the hybrid TKD-CPD. 

As we can see from recent deep learning literature \cite{howard2019searching, tan2019efficientnet, kossaifi2019efficient}, modern state-of-the-art architectures exploit the CP format when constructing blocks of consecutive layers, which consist of $1\times 1$ convolution followed by depth-wise separable convolution.  The intuition that stays behind the effectiveness of such representation is that first $1\times 1$ convolution maps data to a higher-dimensional subspace, where the features are more separable, so we can apply separate convolutional kernels to preprocess them. 
Thus, representing weights in CP format using stable and efficient algorithms is the simplest and efficient way of constructing reduced convolutional kernels.

To the best of our knowledge, our paper is the first work solving a problem of building weights in the CP format that is stable and consistent with the fine-tuning procedure. 

The ability to control sensitivity and stability of factorized weights might be crucial when approaching incremental 
learning task \cite{bulat2019incremental} or multi-modal tasks, where information fusion across different modalities is performed through shared weight factors.

Our proposed CPD-EPC  method can allow more stable fine-tuning of architectures containing higher-order CP convolutional layers \cite{kossaifi2019efficient,kossaifi2019t} that are potentially very promising due to the ability to propagate the input structure through the whole network. We leave the mentioned directions for further research.

\section*{Acknowledgements}

The work of A.-H. Phan, A. Cichocki, I. Oseledets, J. Gusak, 
K. Sobolev, K. Sozykin and D. Ermilov was supported by the Ministry of Education and Science of the Russian Federation under Grant 14.756.31.0001.
%The authors wish to thank Noah's Ark Lab, Huawei Technologies for motivating us in this joint project.
The results of this work were achieved during the cooperation project with Noah's Ark Lab, Huawei Technologies.
The authors sincerely thank the Referees for very constructive comments which helped to improve the quality and presentation of the paper.
The computing for this project was performed on the Zhores CDISE HPC cluster at Skoltech\cite{zhores19}.

%\clearpage
\bibliographystyle{splncs04}
\bibliography{egbib}
\end{document}